\documentclass[letterpaper]{article} 
\usepackage{aaai25}  
\usepackage{times}  
\usepackage{helvet}  
\usepackage{courier}  
\usepackage[hyphens]{url}  
\usepackage{graphicx} 
\usepackage{natbib}  
\usepackage{caption} 
\usepackage{algorithm}

\newif\ifarxiv
\arxivtrue

\usepackage{newfloat}
\usepackage{listings}
\DeclareCaptionStyle{ruled}{labelfont=normalfont,labelsep=colon,strut=off} 
\floatstyle{ruled}
\newfloat{listing}{tb}{lst}{}
\floatname{listing}{Listing}
\title{Exploiting Symmetries in MUS Computation}
\author{
    Ignace Bleukx\textsuperscript{\rm1},
    Hélène Verhaeghe\textsuperscript{\rm1,2},
    Bart Bogaerts\textsuperscript{\rm1,3},
    Tias Guns\textsuperscript{\rm1}
}

\affiliations{
    \textsuperscript{\rm1}KU Leuven, Dept.\ of Computer Science; Leuven.AI, Celestijnenlaan 200A, 3000 Leuven, Belgium \\
    \textsuperscript{\rm2}UCLouvain, ICTEAM, Place Sainte Barbe 2, 1348 Louvain-la-Neuve, Belgium \\
    \textsuperscript{\rm3}Vrije Universiteit Brussel, Dept.\ of Computer Science, Pleinlaan 9, 1050 Brussels, Belgium \\
    ignace.bleukx@kuleuven.be, helene.verhaeghe@uclouvain.be, bart.bogaerts@kuleuven.be, tias.guns@kuleuven.be
}

\usepackage{amsmath}
\usepackage{mathtools}
\usepackage{booktabs}
\usepackage{multirow}
\usepackage{algorithm}
\usepackage[noend]{algorithmic}
\usepackage[dvipsnames]{xcolor}
\usepackage{xspace}
\usepackage{subcaption}

\newcommand\commentmargin[1]{{\small #1}}
\usepackage{marginnote}

 \renewcommand\commentmargin[1]{{\marginpar{#1}}}
  \renewcommand\commentmargin[1]{{{#1}}} 
 \renewcommand\commentmargin[1]{}

\newcommand{\bart}[1]{{\commentmargin{\color{OliveGreen}Bart: #1}}}
\newcommand{\ignace}[1]{{\commentmargin{\color{purple}Ignace: #1}}}
\newcommand{\hve}[1]{{\commentmargin{\color{orange}Hélène: #1}}}

\newcommand{\pageplan}[1]{{\small{\textit{\color{red!50!black} PP: #1}}}}
\renewcommand{\pageplan}[1]{{\marginpar{\small{\textit{\color{red!50!black} PP: #1}}}}}
\renewcommand{\pageplan}[1]{}

\newcommand{\ignore}[1]{}

\usepackage{amsthm}
\ifdefined\beginexample\else
\newtheorem{example}{Example}
\fi

\ifdefined\begindefinition\else
\newtheorem{definition}{Definition}
\fi

\ifdefined\beginproposition\else
\newtheorem{proposition}{Proposition}
\fi

\newtheorem{sidenote}{Side note}

\usepackage{tcolorbox}

\newcommand{\assignment}{\mm{\alpha}}

\newcommand{\cpmpy}{CPMpy\xspace}

\renewcommand{\implies}{\rightarrow}

\usepackage{ amssymb }

\newcommand\formula{\mm{\phi}}
\newcommand\formulahalf{\mm{\indicators \implies \formula}}

\newcommand\partialassignment{\mm{\mu}}

\usepackage{amssymb}
\usepackage{graphicx}

\newcommand{\approxmodels}{\mathrel{\scalebox{1}[1.5]{\mm{\shortmid}}\mkern-3.1mu\raisebox{0.1ex}{$\approx$}}}

\newcommand\mm[1]{\ensuremath{#1}\xspace}
\newcommand\call[1]{\mm{\textsc{#1}}}

\usepackage{wasysym}

\newcommand{\generators}{\mathcal{G}}
\newcommand{\breakers}{\mathcal{B}}
\newcommand{\getbreakers}{\call{GetBreakingConstraints}(\generators)}

\newcommand{\core}{\mm{U}}
\newcommand{\corr}{\mm{C}}

\newcommand{\checkthis}{\mm{c}}
\newcommand{\mus}{\mm{\mathit{MUS}}}
\newcommand{\mcs}{\mm{\mathit{MCS}}}
\newcommand{\issat}{\texttt{is\_sat}\xspace}

\newcommand{\indicators}{\mathcal{A}}

\newcommand{\new}[1]{{\color{teal}#1}} 

\newcommand{\musses}{\mm{\mathit{MUSes}}}
\newcommand{\mcses}{\mm{\mathit{MCSes}}}

\newcommand{\map}{\mm{M}}

\newcommand{\breakid}{\mm{\textsc{BreakID}}}

\newcommand{\quickxplain}{QuickXplain\xspace}

\newcommand{\alllits}{\mm{\mathcal{L}}} 

\newcommand{\php}[2]{\ensuremath{\mathit{php}(#1, #2)}}
\newcommand{\pigeoncons}[1]{\mm{P_{#1}}}
\newcommand{\holecons}[1]{\mm{H_{#1}}}

\newcommand{\constraintsymmetries}[1]{\mm{\call{ConstraintSymmetries}(#1)}}

\newcommand{\sotashrink}{Shrink}
\newcommand{\skipsym}{Symm-Shrink}
\newcommand{\recomp}{Symm-ShrinkR}

\newcommand{\sotaocus}{OCUS}
\newcommand{\lex}{Lex}
\newcommand{\enumsym}{Enum}

\newcommand\setstohit{\mm{H}}
\newcommand\cohs{\mm{\call{CondOptHittingSet}}}
\newcommand\sat{\texttt{{SAT}}}
\newcommand\corrsubsets{\mm{\call{CorrSubsets}}}
\newcommand\subsetT{\mm{S}}
\newcommand\symmocus{\mm{\call{Symm-OCUS}}}

\usepackage{bm}

\newcommand\lit{\mm{\ell}}
\newcommand\var{\mm{x}}

\usepackage{enumitem}
\newlist{RQ}{enumerate}{1}
\setlist[RQ]{label=\bfseries{}RQ\arabic*,leftmargin=30pt}

\newlist{EQ}{enumerate}{1}
\setlist[EQ]{label=\bfseries{}EQ\arabic*,leftmargin=30pt}

\newlist{WP}{enumerate}{1}
\setlist[WP]{label=\bfseries{}WP\arabic*.,leftmargin=30pt}

\usepackage{cleveref}

\begin{document}

\maketitle

\begin{abstract}
    In eXplainable Constraint Solving (XCS), it is common to extract a Minimal Unsatisfiable Subset (MUS) from a set of unsatisfiable constraints.
    This helps explain 
    to a user \emph{why} a constraint specification does not admit a solution.
    Finding MUSes can be computationally expensive for highly symmetric problems, as many combinations of constraints need to be considered.
    In the traditional context of solving satisfaction problems, symmetry has been well studied, and 
    effective ways to detect and exploit symmetries during the search exist.
    However, in the setting of finding MUSes of unsatisfiable constraint programs, symmetries are understudied.
    In this paper, we take inspiration from existing symmetry-handling techniques and adapt well-known MUS-computation methods to
    exploit symmetries in the specification, speeding-up overall computation time.
    Our results display a significant reduction of runtime for our adapted algorithms compared to the baseline 
    on symmetric problems. 
\end{abstract}

\bart{General warning: do not write things like  
	``In \citet{devriendt2012symmetry}'' (unless when talking about intestines of Jo Devriendt). Also do not write 
	``\citet{devriendt2012symmetry} proposes BLA'', but write ``\citet{devriendt2012symmetry} propose''. 
	The subject in these sentences is the set of authors; the ref in between brackets is only there to support the claim. I fixed this in several places. } 

\begin{section}{Introduction}
The field of eXplainable Constraint Solving (XCS) is a subfield of eXplainable AI (XAI) focused on explaining the solutions, or lack thereof, of constraint (optimization) problems.
Explaining \emph{why} a set of constraints does not admit a solution is often done through a \emph{Minimal Unsatisfiable Subset} (MUS), i.e., an irreducible subset of the constraints rendering the problem unsatisfiable.
This subset of constraints is then easier for a user to analyze than the full problem.
MUSes are also used for debugging constraint models \cite{leo2017debugging}, including minimization of faulty models when fuzz testing~\cite{paxian2023maxsat}, explaining \emph{why} an objective value is optimal \cite{bleukx2023simplifying} or to explain \emph{why} a fact follows from the constraints \cite{bogaerts2021framework}.

MUS computation techniques are well-studied and well-known in the XAI literature.
Frequently used algorithms can be classified into \emph{shrinking}\footnote{sometimes called \emph{destructive} or \emph{deletion-based}} methods \cite{marques2010minimal}; \emph{divide-and-conquer} algorithms such as QuickXplain \cite{junker2001quickxplain} and \emph{implicit-hitting-set based} methods \cite{ignatiev2015smallest}.

In some cases, it can be useful to not just compute a single MUS but to \emph{enumerate} a collection of MUSes or even all of them, as some MUSes may be easier understood by a user than others.
Many algorithms for computing a set of MUSes rely on the ``seed-and-shrink'' paradigm~\cite{bendik2020must, bendik2020replication}, where variants of the MARCO-algorithm are among the most popular techniques~\cite{liffiton2016fast}.
MUS-computation and enumeration techniques have been discussed extensively \cite{silva2020inconsistent,dev2021explanation}.

A prime concern for MUS computation and enumeration techniques is efficiency, especially for large problems, as checking if a valid assignment exists for a set of constraints can already be NP-hard \cite{DBLP:series/faia/336}. 
In practice, many techniques exist to speed up solving, e.g., by exploiting the problem structure. 
One of these, which has barely been studied in the context of MUS computation, is the exploitation of symmetries in the problem formulation.

Many real-world problems exhibit some kind of (variable and/or value) symmetry.
For example, packing items into equivalent trucks, assigning shifts to equivalent workers or scheduling tasks on equivalent machines.
Such symmetries can slow down combinatorial solvers, as they might have to consider all symmetric alternatives to an assignment.
Solvers can exploit symmetries to prune the search space and/or speed up the search \cite{gent2006symmetry,DBLP:series/faia/Sakallah21}.
Symmetry exploitation techniques are either \emph{static} or \emph{dynamic}.
Static techniques involve adding \emph{symmetry breaking constraints} to the specification before starting the solver.
Dynamic techniques exploit the symmetries during search, e.g., by automatically learning symmetric clauses~\cite{devriendt2012symmetry, chu2014lazy, mears2014lightweight,devriendt2017learning}, by modifying their branching behaviour~\cite{torsten2001symmetry, sabharwal2005symchaff} or by generating symmetry breaking constraints during the solvers' execution~\cite{metin2018cdclsym}.
Furthermore, several tools exist for automatically detecting symmetries in the constraint specifications for a variety of input formats~\cite{drtiwa11a,breakid2016,daimy22,ABR24SatsumaStructure-BasedSymmetryBreakingSAT}.

While symmetry-handling has extensively been studied for methods solving constraint satisfaction or optimization problems, they have barely been studied in the context of MUS computation or enumeration. 
``Classic'' symmetries are defined on assignments of variables, while MUS-computation algorithms reason over subsets of constraints. 
Still, symmetries over variables can also lead to symmetries over constraints and, hence, over MUSes.

In this paper, we build on this observation and investigate how to discover and exploit symmetries for MUS computation and enumeration. 
Our contributions are the following:
    (i) we formally define symmetries in MUS problems;
    (ii) we show how existing symmetry detection tools can be used to detect constraint symmetries by means of a reformulation with half-reified constraints
    \bart{weird terminology. I would go for ``blocking variables'' or ``assumption variables''. Also: this terminology is not used in the rest of the paper??? Later in the paper you DO use assumption variables, indicator variables, indicator constraints, half-reification... I would say: pick a single name and use it consistently} \hve{agreed, I would be in favor of assumption variables as it seems the most generic term}
    (iii) we show how to modify different types of MUS-computation algorithms to speed up the search for an MUS; and
   (iv) we evaluate the potential runtime improvements of symmetry breaking for MUS methods in an elaborate experimental evaluation.
\end{section}

\begin{section}{Background} \label{sec:background}
The methods presented in this paper are defined on constraints with Boolean variables only, but they can be generalized to richer constraint-solving paradigms such as SMT, MIP, or CP.
In this section, we recall several essential concepts and introduce the notation used throughout this paper.

A literal \lit is a Boolean variable $\var$ or its negation $\neg \var$.
We use \formula to represent a set of constraints over Boolean variables.
The set of literals in a constraint specification $\alllits(\formula)$ contains all variables occurring in \formula and their negation.
When \formula is clear from the context, we simply write \alllits.
An assignment \assignment maps Boolean variables in \formula to either true or false.
An assignment is \emph{total} if it assigns every variable in \formula, and \emph{partial} otherwise.
Assignments can be represented as a subset of literals, namely those assigned to true \cite{audemard2018glucose}. 
When a partial assignment \partialassignment can be \emph{expanded} to a total one while satisfying \formula, we write $\partialassignment \approxmodels \formula$.

Modern SAT-solvers allow the use of assumption variables \cite{nadel2012efficient, hickey2019speeding}.
In combination with implication constraints, they can be used to test whether a subset of constraints is satisfiable.
For each constraint $c$ in \formula, we introduce a Boolean indicator variable $a_c$ and use this variable to construct the half-reification of $c$: $a_c \implies c$ \cite{feydy2011half}.
We refer to the set of indicators as $\indicators$ and, abusing notation, we refer to the constraint specification \formula where all constraints are half-reified using variables $a_c$ as $\formulahalf$.
Now, any full assignment of $\indicators$ (which is also a partial assignment of $\formulahalf$), can be interpreted as a subset of constraints.
Namely, a constraint $c$ is part of the subset if its indicator variable $a_c$ is set to true.
So, when assuming a subset of literals $\indicators$ to be true, we can test whether a subset of constraints admits a valid assignment using a SAT-solver.
Moreover, if no such assignment exists, SAT-solvers can return a sufficient set of assumption variables that cause unsatisfiability.
Hence, we assume to have an oracle that takes as input a set of constraints $\formula$ and which returns a tuple $(\issat, \assignment, \core)$.
Here, \issat is a Boolean flag indicating whether $\formula$ is satisfiable.
When \issat is set to true, \assignment is the assignment found by the oracle, and $\core \subseteq \formula$ is a sufficient subset of constraints causing unsatisfiability otherwise.

Constraint specifications can exhibit \emph{symmetries}.
We distinguish two types of symmetries: syntactic and semantic.

\begin{definition}[Syntactic symmetry~\cite{breakid2016}] \label{def:syntactic}
Let $\pi$ be a permutation of all literals $\alllits$ in constraints \formula.
A \emph{syntactic symmetry} of \formula, is a permutation $\pi$ that commutes with negation (i.e., $\pi(\neg l) = \neg \pi(l)$), and that when applied to the literals in each of the constraints in \formula, maps \formula to itself.
\end{definition}

\begin{definition}[Semantic symmetry~\cite{breakid2016}]\label{def:semantic}
Let $\pi$ be a permutation of all literals $\alllits$ in constraints \formula.
A \emph{semantic symmetry} of \formula is a permutation $\pi$ that \emph{commutes with negation} and that \emph{preserves satisfaction} to $\phi$ (i.e., $\pi(\alpha)$ satisfies $\phi$ iff $\alpha$ satisfies $\phi$). 
\end{definition}

Any syntactic symmetry is also a semantic symmetry but not vice versa \cite{DBLP:series/faia/Sakallah21}.
In practice, most symmetry-detection tools only detect syntactic symmetries~\cite{breakid2016}, but all concepts in this paper are valid for semantic symmetries too, unless specified otherwise.
Therefore, we refer to ``symmetries'' in general in the remainder of this paper.
We write permutations using disjoint cycle notation.
E.g., $\pi = (abc)(de)$ denotes the permutation with: $\pi(a)=b, \pi(b)=c, \pi(c)=a, \pi(d)=e$ and $\pi(e)=d$

Some structured symmetry groups can be summarized as a row-interchangeable symmetry \cite{flener2002breaking,devriendt2014breakidglucose}.
\begin{definition}[Row-interchangeability]
A matrix $M=(x_{rc})$ of literals describes a row-interchangeability symmetry group if for each permutation  $\rho$ of its rows, $\pi^{M}_{\rho}: x_{rc} \mapsto x_{\rho(r)c}$ 
is a symmetry of \formula.
\end{definition}

The last required concept in this paper is that of Minimal Unsatisfiable Subsets (MUS) \cite{marques2010minimal}.
\begin{definition}[Minimal Unsatisfiable Subset]
A Minimal Unsatisfiable Subset (\mus) of a set of constraints $\phi$ is a set $U \subseteq \phi$ that is unsatisfiable and for which any proper subset $\core' \subsetneq \core$ is satisfiable.
\end{definition}

Informally, an MUS is a minimal set of constraints that renders the problem unsatisfiable.
Note that there may be several MUSes for a given unsatisfiable constraint problem.
\end{section}

\begin{section}{Symmetries for the MUS problem}
In this section, we define symmetries of the MUS problem.
Solving the MUS problem requires reasoning over subsets of constraints, whereas the traditional setting of solving a satisfaction problem requires reasoning over assignments.
Hence, symmetries for the MUS problem are symmetries of \emph{constraints}.

Note that any syntactic symmetry of variables in a set of constraints induces a symmetry of constraints \cite{cohen2006definitions}, as illustrated below.

\begin{example}[Pigeon hole problem]
\label{example:php}
Consider a pigeon-hole problem $\php{p}{h}$ with \mm{p} pigeons and \mm{h} holes:  
\begin{align}
    & \sum\nolimits_{j = 1}^{h} x_{ij} \geq 1 & \forall i \in \{1,\dots,p\} \tag{\pigeoncons{i}} \\
    & \sum\nolimits_{i = 1}^{p} x_{ij} \leq 1 & \forall j \in \{1,\dots,h\} \tag{\holecons{j}}
\end{align}

Where $x_{ij}$ are Boolean variables indicating whether pigeon \mm{i} is assigned to hole \mm{j}.
We refer to the constraint that ensures pigeon \mm{i} is in a hole as \pigeoncons{i} and the constraint ensuring at most one pigeon is in hole \mm{j} is refered to as \holecons{j}.

Considering $\php{4}{2}$, we identify four MUSes:
\begin{align*}
\{\pigeoncons{1}, \pigeoncons{2}, \pigeoncons{3}, \holecons{1}, \holecons{2}\}, 
\{\pigeoncons{1}, \pigeoncons{2}, \pigeoncons{4}, \holecons{1}, \holecons{2}\} \\
\{\pigeoncons{1}, \pigeoncons{3}, \pigeoncons{4}, \holecons{1}, \holecons{2}\}, 
\{\pigeoncons{2}, \pigeoncons{3}, \pigeoncons{4}, \holecons{1}, \holecons{2}\} 
\end{align*}
\end{example}
In the above specification, we notice several syntactic symmetries. E.g., $(x_{11}x_{21})(x_{12}x_{22})(x_{13}x_{23})(x_{14}x_{24})$ which in turn induces a symmetry of constraints: $(\pigeoncons{1}\pigeoncons{2})$. 


\subsection{Computing constraint symmetries} 
We propose to use existing symmetry-detection tools such as \breakid~\cite{breakid2016}, which traditionally detect \hve{syntactic? (to be precise, as above it is said that when speaking about symmetries we speak about both interchangeably?)} \ignace{Yes, but breakid does syntactic+, so a little bit more than purely syntactic as far as I understand} symmetries of variable assignments, and transform the input to these tools to detect symmetries in constraints.
In particular, we introduce indicator variables $a_c$ to construct $\formulahalf$ as defined in \Cref{sec:background}.

Before introducing symmetries that can be used when computing MUSes, we generalize the concept of symmetries of assignments to symmetries of partial assignments.


\begin{definition}[Partial symmetries]\label{def:partialsym}
Let $\subsetT \subseteq \alllits(\formula)$ be a subset of all literals $\alllits(\formula)$, closed under negation. 
A partial \subsetT-symmetry of a specification $\phi$ is a permutation of $\subsetT$ that \emph{commutes with negation} and that preserves satisfiability to $\phi$. 
That is, for any assignment $\mu$ of $\subsetT$, $\mu \approxmodels \formula$ iff $\pi(\mu) \approxmodels \formula$.
\end{definition}

Some tools allow us to search for partial symmetries explicitly, but they can be derived from ``classical symmetries'' as in Definitions~\ref{def:syntactic} and \ref{def:semantic}, as the following holds.

\begin{proposition}[Deriving partial symmetries]
Let $\subsetT$ be a subset of $\alllits(\formula)$ closed under negation and $\pi$ a symmetry of $\formula$.
If $\pi(\subsetT) = \subsetT$, then $\pi|_\subsetT$ is a partial \subsetT-symmetry of $\formula$.   
\end{proposition}
\bart{Proof is missing in appendix! Not needed in final version but we should put this in the Arxiv version that has all the appendices!}

We are now ready to fully define the concept of constraint symmetries as used in this paper.

\begin{definition}[Constraint symmetry]
	A permutation $\pi$ of the constraints in $\formula$ is called a \emph{constraint symmetry} for $\formula$ if for any subset $C \subseteq \formula$, $\pi(C)$ is satisfiable iff $C$ is satisfiable.
\end{definition}

The following proposition describes the exact relation between partial-symmetries and constraint symmetries.
\begin{proposition}
Let $\pi$ be a permutation of variables (that is a permutation of literals that does not cross polarity) $\indicators$ in $\indicators \implies \formula$, then $\pi$ directly maps to a permutation of constraints $\pi_\formula$.
Under this mapping, $\pi$ is an $\indicators$-symmetry iff $\pi_\formula$ is a constraint symmetry.
\end{proposition}

Intuitively, a subset of indicators ``enables'' a set of constraints and, when this initial set is (un)satisfiable, so is the set of constraints that is enabled by their symmetric image.


\bart{might also turn this into a small propsition: any permutation of assumptions $\pi$ directly maps to a permutation of constraints $\pi_C$. 
Under this mapping $\pi$ is an $\indicators$-symmetry iff $\pi_C$ is a constraint symmetry.} 
\bart{But technically speaking... there are $A$-symmetries that do not really map to constraint symmetries because the indicators migth be mapped to the negation of another indicators (so the precise correspondence is not super clear yet). In practice, this will nto happen since the assumption vars only appear in one polarity and we detect syntactic symms}
\bart{Above comment was here before submission already. But certainly: it is not \textbf{CLEAR} that this correspondence holds, so should be phrased more carefully!! } 
Hence, by computing partial symmetries of the specification $\formulahalf$, we can compute constraint symmetries using any existing tool supporting it.
%
We refer to the above method of finding constraint symmetries as $\constraintsymmetries{\formula}$.

Following \Cref{prop:symmus}, constraint symmetries map MUSes to MUSes and non-MUSes to non-MUSes.

\begin{proposition}\label{prop:symmus}
Given a constraint symmetry $\pi$ of $\formula$ and a subset $\core \subseteq \formula$.
Then, $\pi(\core)$ is an MUS of \formula if and only if $\core$ is an MUS of \formula.
\end{proposition}



\end{section}

\begin{section}{Constraint symmetries in MUS algorithms}
Once detected, symmetries can be exploited to speed up the search for valid assignment(s) that satisfy $\phi$.
Static techniques include adding \emph{symmetry breaking constraints} to the constraints specification before solving.
Such constraints can speed-up the search by excluding a part of the search space.
A well-known symmetry breaking constraint is the Lex-Leader constraint, which excludes, given an order of the variables, any assignment~$\alpha$ if its symmetric counterpart $\pi(\alpha)$ is lexicographically below $\alpha$. 
Symmetry detection tools can often generate a set of breaking constraints along with the generators for each detected symmetry group.
Symmetry breaking constraints may break the symmetry completely or only partially \cite{mcdonald2002partialbreaking}.

Dynamic symmetry handling involves modifying the search algorithm of the solver itself to exploit the symmetries in the constraints.
For example, by avoiding symmetric branches in the search tree or by modifying propagation algorithms to exploit symmetries in the problem specifically.
While symmetry handling techniques have been extensively studied in the context of solving constraint programs, they do not directly apply to finding MUSes.

In particular, by adding a set of symmetry breaking constraints $\breakers$ to constraints \formula, any MUS of the ``broken specification'' will be a subset of $\formula \cup \breakers$.
This means that $\formula \cup \breakers$ can contain more MUSes than the original set of constraints, and new MUSes do not necessarily map to an original MUS.
\ignore{
E.g., when selecting a constraint forcing pigeons with lower indices to be in holes with lower indices ($\in \breakers$), selecting the constraint $\pigeoncons{2}$ as part of an MUS, implicitly selects $\pigeoncons{1}$ as pigeon 2 can only be in a hole if pigeon 1 has a hole assigned.
}

In this section, we generalize symmetry handling techniques to constraint symmetries for use in MUS-finding algorithms.
We explore both static symmetry breaking and techniques to exploit symmetries dynamically.

\subsection{Symmetric transition constraints}
\label{sec:dynamic-shrink}

One of the simplest classes of algorithms for finding an MUS are ``shrinking'' based methods \cite{marques2010minimal, wieringa2014incremental}.
These methods iteratively drop a constraint \checkthis from \formula, and call an oracle to check whether the remainder is (un)satisfiable.
If the remaining core is still UNSAT, the constraint can safely be dropped from the formula.
Otherwise, the constraint is required to ensure the unsatisfiability of the core and is marked as such (line~\ref{line:mark} in \Cref{algo:shrink}).
When a constraint is marked as required, it is called a \emph{transition constraint}~\cite{belov2011accelerating}.

\begin{definition}[Transition constraint]
Given an unsatisfiable set of constraints $\core$ containing constraint $c$.
If $\core \setminus \{c\}$ is satisfiable, $c$ is called a \emph{transition constraint}.   
\end{definition}

\citet{belov2012towards} propose an optimization to this simple approach called \emph{clause set refinement}.
This technique exploits the core found by the solver after it decides the input is unsatisfiable.
In particular, the core returned by the solver may be smaller than the working core at that point in the algorithm.
Hence, we can use the solver core $\core'$ to further shrink the working core.
This technique is shown on line~\ref{line:csr} in \Cref{algo:shrink}.

To exploit symmetries in shrinking-based methods, we use the set of symmetry generators directly.
\bart{Algo1does not work recursively? 
	If $c$ is marked and so is $\pi(c)$, you will not check for $\sigma(\pi(c))$?  You really only consider the generators?}
 \ignace{Yes, indeed only the generators in the implementation}
 \bart{This is a design decision! You might watn to discuss or at least mention it somewhere. }
 \ignace{Mentioned now in the paragraph about how to find symmetric transition constraints efficiently}
In particular, we mark symmetric counterparts of transition constraints.
We illustrate this in \Cref{example:deletion}.

\begin{example}\label{example:deletion}
Following on~\Cref{example:php},
take U to be $\{\pigeoncons{1}, \pigeoncons{2}, \pigeoncons{3}, \holecons{1}, \holecons{2}\}$ and \mm{\holecons{1}} and \mm{\holecons{2}} are the only constraints that are marked to be required by previous iterations.
If the algorithm selects $\pigeoncons{1}$ to be the next constraint to test, the oracle will report $\core \setminus \{\pigeoncons{1}\}$ is satisfiable.
Hence, \mm{\pigeoncons{1}} is a transition constraint.
However, as all \mm{\pigeoncons{i}} are interchangeable for the MUS problem, this means also \mm{\pigeoncons{2}} and \mm{\pigeoncons{3}} are required, and they can be marked as such.  
\end{example}

\begin{algorithm}[t]
\small
\begin{algorithmic}[1]
 \caption{$\call{Symm-Shrink}(\formula, \new{\mathit{recompute?}})$ }\label{algo:shrink}
        \STATE $\core \gets \formula$; \new{$\generators \gets \constraintsymmetries{\formula}$}
        \WHILE{there are unmarked constraints in \mm{\core}}
            \STATE $\checkthis \gets$ next unmarked constraint in \mm{\core} \;
            \STATE $(\issat, \assignment, \core') \gets \sat(\core \setminus \{\checkthis\})$
            \IF{\mm{\issat}}
                \STATE mark $\checkthis$ as required \label{line:mark}
                \new{
                \IF{$\mathit{recompute?}$}
                    \STATE $\generators \gets \constraintsymmetries{\core}$
                \ENDIF
                \FOR{ each $\pi \in \generators$}
                    \IF{$\pi(\core) = \core$} \label{line:checksame}
                        \STATE mark $\pi(\checkthis)$ as required \label{line:marksym}
                    \ENDIF
                \ENDFOR
                
                }
            \ELSE
                \STATE $\core \gets \core'$ \label{line:csr}
            \ENDIF
        \ENDWHILE
        \RETURN{$\core$}
\end{algorithmic}
\end{algorithm}
To mark such symmetric transition constraints, we are interested in constraint symmetries mapping \core to \core.
\begin{proposition}[Symmetric transition constraint]
If $\pi$ is a constraint symmetry of unsatisfiable formula $\core\subseteq\formula$ and $c$ a transition constraint for $\core$, then $\pi(c)$ is also a transition constraint for $\core$.
\end{proposition}
\begin{proof}
    As $\core$ is unsatisfiable and $\pi$ is a constraint symmetry of $\core$, $\pi(\core)$ is a also unsatisfiable.
    Furthermore, as $\core \setminus \{c\}$ is satisfiable, $\pi(\core \setminus \{c\}) = \pi(\core) \setminus \{\pi(c)\} = \core \setminus \{\pi(c)\}$ is satisfiable and hence $\pi(c)$ is a transition constraint.
\end{proof}
Finding constraint symmetries of \core can be done either by iteration over each permutation in the symmetry groups in $\generators$ or by invoking the symmetry-detection tool again on the sub-problem \core
(when flag $\mathit{recompute?}$ is true).
The efficiency depends on the overhead of detecting symmetries and the structure of the global symmetry groups.
Indeed, for large symmetry groups, iteration over each permutation may be infeasible or inefficient when only a subset of permutations map \core to \core.
However, some symmetry-detection-tools can summarize structured symmetry groups using a matrix.
For those matrices, we can easily find all of the symmetric images of a given variable within a given subset.
Given a row-interchangeable symmetry described by matrix \mm{M}, a core $\core$ and a (indicator) variable $a \in \core$.
Find the coordinate $(i,j)$ of variable $a$ in \mm{M}.
Now, find the indices of columns $\mathit{cols}$ where, on row $i$, variables in $\core$ occur -- clearly $j \in \mathit{cols}$.
Then, iterate over each each row $r$ in the matrix and define $\mathit{cols}_r$ to be the columns $k$ where $x_{rk} \in \core$. When $\mathit{cols} = \mathit{cols_r}$, then $x_{rj}$ is a symmetric image of $a$ in $\core$.

To further reduce the overhead, we only project to symmetric transition constraints using the generators.

Note that our modification exploits symmetries without ``breaking'' them.
That is, any MUS that may be returned by \call{Shrink} may also be returned by \call{Symm-Shrink}.

Our approach shares similarities with \emph{model rotation} \cite{belov2014clausal}.
Given a satisfiable subset of constraints to check $\core \setminus \{\checkthis\}$, model rotation exploits the assignment \mm{\assignment} found by the oracle to $\core \setminus \{\checkthis\}$.
In particular, it searches for a variable assignment in \mm{\assignment}, such that, when the assignment is changed, \checkthis is satisfied and \emph{exactly one} other constraint $\checkthis' \in \core$ becomes unsatisfied.
When such a variable assignment is found, we can be certain \mm{\checkthis'} is a transition constraint for the MUS and marked as such.

Clearly, model rotation implicitly captures ``simple'' symmetries and hence resembles our approach.
However, in $\call{SymmShrink}$, the solution to $\core \setminus \pi(\checkthis)$ may contain any number of changed variable assignments, whereas model rotation only changes a single assignment.
Moreover, efficient model rotation requires the concept of a ``flip-graph'' which, to the best of our knowledge, can only be constructed easily for clausal input.
As the algorithms described in this paper are more broadly applicable, we do not consider shrinking-based methods using model rotation as the baseline.
In the future, we aim to explore the differences and similarities between model rotation and our approach further.

\subsection{Symmetry breaking in MUS-computation}
\label{sec:static-ocus}

In this section, we investigate how symmetry breaking can be used to speed-up MUS computation.
That is, instead of finding any MUS, we search for a lex-minimal MUS.

\begin{example}[Lex-minimal MUS]
Given the set of MUSes from \Cref{example:php}, and the following constraint order:
$$
    \pigeoncons{1} < \pigeoncons{2} < \pigeoncons{3} < \pigeoncons{4} < \holecons{1} < \holecons{2}
$$
Then $\{\pigeoncons{1}, \pigeoncons{2}, \pigeoncons{3}, \holecons{1}, \holecons{2}\}$ is the only lex-minimal MUS.
\end{example}

Lex-minimal MUSes can be computed using the \quickxplain algorithm.
QuickXplain takes as input a partial ordering of constraints and computes a preferred MUS based on that ordering.
Hence, when providing an ordering mapping to the lex-leadership relation, \quickxplain finds a lex-minimal MUS.
However, such ordering is often not reported by symmetry detection tools as they internally use the order to construct a set of breaking constraints, which can be used for static symmetry breaking.
Therefore, we do not consider the \quickxplain algorithm here and instead focus on methods that can directly use symmetry breaking constraints returned by the detection tool.
In particular, instead of finding any MUS, we search for an Optimal Constrained Unsatisfiable Subset (OCUS) \cite{gamba2023efficiently}.
\begin{definition}[OCUS]
Given a set of constraints \formula, a cost function $f: 2^\formula \rightarrow \mathbb{N}$ and predicate $p: 2^\formula \rightarrow \{\mathit{false}, \mathit{true}\}$.
Then $\core \subseteq \formula$ is an OCUS with respect to $f$ and $p$ if:
\begin{itemize}
    \item \core is unsatisfiable
    \item $p(\core)$ is true
    \item for all other unsatisfiable subsets $\core' \subseteq \formula$ for which $p(\core')$ holds, $f(\core) \leq f(\core')$
\end{itemize}
\end{definition}

In this paper, we use the concept of an OCUS to find an MUS with minimal cardinality by using a linear cost function $f$ with equal weights.
%
Combined with a set of constraint symmetry breaking constraints as predicate $p$, this ensures the resulting OCUS is a smallest, lex-minimal MUS of \formula.

The only OCUS computation algorithm we are aware of \cite{gamba2023efficiently} is a modification of the well-known \emph{smallest MUS} algorithm, which is based on the hitting-set dualization between MUSes and MCSes \cite{ignatiev2015smallest}.

\begin{definition}[Minimal Correction Subset]
    Given a set of constraints $\formula$, a Minimal Correction Subset (\mcs) is a subset $\corr \subseteq \formula$ for which $\formula \setminus \corr$ is satisfiable and for which any proper subset $\corr' \subsetneq \corr$ it holds that $\formula \setminus \corr'$ is unsatisfiable.
\end{definition}

Informally, an MCS is a minimal set of constraints which, when relaxed, render the problem satisfiable.

\begin{proposition}[Hitting set dualization \cite{ignatiev2015smallest}]
\label{prop:hittingset}
    Given a set of constraints \formula, let $\mathit{MUSes}(\formula)$ and $\mathit{MCSes}(\formula)$ be the set of all MUSes and MCSes of \formula, respectively. Then, the following holds
    
\begin{enumerate}
    \item A subset \core of \formula is an MUS if and only if \core is a minimal hitting set of $\mcses({\formula})$; and 
    \item A subset \corr of \formula is an MCS if and only if \corr is a minimal hitting set of $\musses({\formula})$.
\end{enumerate}
\end{proposition}

The Implicit-Hitting-Set (IHS) algorithm to find an OCUS keeps a set $\setstohit$ of correction subsets (initially empty).
In each iteration, an optimal and constrained hitting set of $\setstohit$ is computed, and the satisfiability of this set is checked using an oracle.
If the hitting set is satisfiable, one or more correction subsets are generated from it using \Cref{algo:corrsubsets} and are added to the sets to hit $\setstohit$.

This process is repeated until the optimal hitting set is UNSAT, and therefore, is an MUS.
The pseudo-code for this algorithm is shown in \Cref{algo:symmocus} (the $\mm{dynamic?}$ parameter will be explained later).

\begin{algorithm}[t]
\small
\caption{$\symmocus(\formula , f, \new{\mm{dynamic?}})$} 
\label{algo:symmocus}

\begin{algorithmic}[1]
    \STATE $\setstohit  \gets \emptyset$; \new{$\generators \gets \call{ComputeSymmetries}(\formula)$; \\$\breakers \gets \call{GetLexLeaderConstraints}(\generators)$  }
    \WHILE{true}
        \STATE $\mm{S} \gets \call{\new{Symm-}}\cohs(\setstohit,f, \new{\breakers})$ \label{line:opt}
        \STATE $(\issat, \assignment, \core) \gets \sat(S)$
        \IF{$\neg \issat$} \label{alg:ocus-sat-check}
            \RETURN $(\mm{S}, \mathit{status})$
        \ENDIF
        \IF{\new{\mm{dynamic?}}}
            \new{
           \STATE $\mm{K} \gets \corrsubsets(\subsetT, \formula, \new{\generators})$ \label{line:dynamic-grow}
            }
        \ELSE
           \STATE $\mm{K} \gets \corrsubsets(\subsetT, \formula)$ \label{line:grow}
        \ENDIF
        \STATE $\setstohit  \gets \setstohit  \cup \mm{K} $  \label{alg:ocus:complement}
    \ENDWHILE
\end{algorithmic}
\end{algorithm}

To illustrate how breaking symmetries in IHS-algorithms may be useful, consider the following example.

\begin{example}[Lex-minimal hitting set]
Consider again $\php{4}{2}$.
After computing constraint symmetries, we can construct the following symmetry breaking constraints:
$
\{\pigeoncons{1} \vee \neg\pigeoncons{2}, 
 \pigeoncons{2} \vee \neg\pigeoncons{3}, 
 \pigeoncons{3} \vee \neg\pigeoncons{4}, 
 \holecons{1} \vee \neg\holecons{2}
\}$
which enforce a preference for lower-indice pigeon constraints and hole constraints.
Imagine during the run of the algorithm, the sets to hit (i.e., the correction subsets enumerated by previous iterations) are 
$\{\holecons{2}\}$ and $ \{\pigeoncons{3}, \pigeoncons{4}\}$

Then, a minimal hitting set is $\{\holecons{2}, \pigeoncons{3}\}$, which is satisfiable and thus clearly not an MUS.
However, the smallest set that satisfies the breaking constraints, and that hits $\{\holecons{2}\}$ and $\{\pigeoncons{3}, \pigeoncons{4}\}$ is $\{\pigeoncons{1}, \pigeoncons{2}, \pigeoncons{3}, \holecons{1}, \holecons{2}\}$.
This subset is UNSAT and, indeed, an MUS, and hence the algorithm terminates.
\end{example}

\subsection{Enumeration of symmetric MCSes}
\label{sec:dynamic-ocus}
Implicit Hitting Set algorithms for computing MUSes are based on the hitting set dualization as stated in~\Cref{prop:hittingset}.
They construct the set of minimal correction subsets lazily, by iteratively calling a hitting set solver and computing one or more correction subsets from the given hitting set.
In general, the more correction subsets can be added in each iteration of the algorithm, the fewer iterations are required to find an MUS. 
Indeed, as shown by \citet{ignatiev2015smallest}, enumeration of disjoint MCSes using for example \Cref{algo:corrsubsets}, significantly improves the algorithm's runtime.

\begin{algorithm}[t]
\small
\caption{$\corrsubsets(\subsetT, \formula, \new{\generators})$} \label{algo:corrsubsets}
\begin{algorithmic}[1]
\STATE $\mm{K} \gets \emptyset$; $S' \gets \subsetT$; $(\issat, \alpha, \core) \gets \sat(S')$
\WHILE{\issat}
    \STATE $C \gets \{c \mid c \in \formula, c~\text{is unsatisfied by}~\alpha\}$
    \new{
    \FOR{$\pi \in \generators$}
        \STATE $S' \gets S' \cup \pi(C)$
        \STATE $K \gets K \cup \{\pi(C)\}$
    \ENDFOR
    }
    \STATE $(\issat, \alpha, \core) \gets \sat(S')$
\ENDWHILE
\RETURN $K$
\end{algorithmic}
\end{algorithm}

Our contribution in this section also adds more MCSes in a single iteration of the algorithm.
We propose to use the symmetry-groups detected directly by enumerating symmetric images of correction subsets (line~\ref{line:dynamic-grow} of \Cref{algo:symmocus} when the $\mathit{dynamic?}$ flag is enabled).

We illustrate our idea in the following example:

\begin{example}
\label{example:dynamic-ocus}
    Take again the running example of $\php{4}{2}$.
    Imagine the sets to hit at some point in the algorithm are:
    $$
        \{H_1\}, \{P_1, P_2\}
    $$
    Then, a minimal hitting set is $\{H_1, P_1\}$, and from this subset, we can compute a minimal correction subset, for example, $\{P_2, P_3\}$.
    As all pigeon constraints are interchangeable, we can find the symmetric versions of the above MCS:
    $$
        \{P_1, P_2\}, \{P_1, P_3\}, \{P_1, P_4\}, \{P_2, P_4\}, \{P_3, P_4\}
    $$
    By adding all of these correction subsets to the sets to hit, the hitting set computed in the next iteration of the algorithm is guaranteed to include at least 3 out of the 4 pigeon constraints, as required in any MUS for this problem.    
\end{example}

Note that the number of symmetric MCSes may be exponential.
Therefore, enumerating \emph{all} of the MCSes can actually slow down the algorithm instead.
Either because just enumerating them is hard, or because the hitting set solver is slowed down significantly by the surplus in sets to hit.
In the experimental section, we evaluate several settings of the algorithm to investigate good values for the upper bound on the number of symmetric MCSes to add in each iteration.

In the implementation of the algorithm, we avoid re-generating MCSes by keeping track of the global set $\setstohit$.

The enumeration of extra MCSes and the addition of symmetry breaking constraints (Section~\ref{sec:static-ocus}) are complementary.
That is, for some hitting sets, the lex-leader constraints do not yield a larger hitting set (e.g., such as the one in Example~\ref{example:dynamic-ocus}).
However, as discussed above, the enumeration of MCSes also comes at a cost, so it might be better for some problems to use the lex-leader constraints in the hitting-set solver instead.
The combination and comparison of both techniques are presented in the experimental section.

\subsection{Symmetries for MUS-enumeration}

\begin{algorithm}[t]
\small
\begin{algorithmic}[1]
 \caption{$\call{Lex-MARCO}(\formula)$ } \label{algo:marco}
        \STATE $\map \gets \emptyset $; \new{$\generators\gets \call{ComputeSymmetries}(\formula)$; \\$\breakers \gets \getbreakers$}
        \WHILE{\map is satisfiable}
            \STATE $\subsetT \gets \call{\new{Symm-}GetUnexplored}(\map, \new{\breakers})$ \label{line:marcomap}
            \STATE $(\issat, \alpha, \core) \gets \sat(\subsetT)$
            \IF {\issat}
                \STATE $\mm{C} \gets \call{Grow}(\subsetT, \formula)$
                \label{line:marcogrow}
                \STATE Block down \mm{C}
            \ELSE
                \STATE $\core' \gets \call{Shrink}(\core)$
                \STATE Block up $\core'$
            \ENDIF
        \ENDWHILE
        \RETURN{$\core$}
\end{algorithmic}
\end{algorithm}

The previous sections describe modifications to algorithms for computing \emph{one} MUS.
In some applications, users may be interested in a collection of MUSes instead.
To this end, \citet{liffiton2016fast} proposed the MARCO algorithm (\Cref{algo:marco}). 

Similar to the IHS algorithm presented before, the MARCO algorithm is based on the hitting set dualization of MUSes and MCSes (\Cref{prop:hittingset}) and explores the power-set of all constraints in an efficient way.
In particular, instead of calculating a \emph{minimal} hitting set to the set of correction subsets, MARCO computes \emph{any} hitting set (line~\ref{line:marcomap} in~\Cref{algo:marco}), which is called the \emph{seed}. 
Next, a SAT oracle is invoked to check whether the computed seed is satisfiable.
If this is the case, the seed is used to calculate a correction subset, which is then added to the hitting-set solver.
When the seed is unsatisfiable, it is shrunk further down to an MUS, and the hitting-set-solver is instructed to exclude any super-set of the found MUS as next seeds.

When many symmetries are present in the constraint specification, the problem may contain many MUSes that are similar from a user perspective \cite{LeoGB024}.
To reduce the cognitive load on a user, we propose to use symmetry breaking constraints in the map-solver to specify a preference on the seed to compute from the map.
This modification will only generate seeds that adhere to the lex-leadership relation defined by the symmetry breaking constraints, and hence the number of MUSes enumerated by the algorithm is reduced.
When the full set of MUSes is required for the application at hand, it can be reconstructed in post-processing based on \Cref{prop:symmus}.
This is similar to how symmetries are used when counting solutions to a satisfaction problem \cite{wang2020counting}.

Note that any method may be used to grow and shrink the seed in the MARCO algorithm.
Therefore, any symmetry-related optimizations to either shrinking or growing-algorithms may directly benefit the performance of MARCO as well.
E.g., for shrinking, one can use \call{\new{Symm-}Shrink} (\Cref{algo:shrink}) 
and for growing similar symmetry-inspired techniques can be devised. 

\end{section}

\begin{section}{Experiments}
\label{sec:experiments}

\begin{figure*}[t]
\begin{subfigure}[c]{0.33\textwidth}
\includegraphics[width=\textwidth]{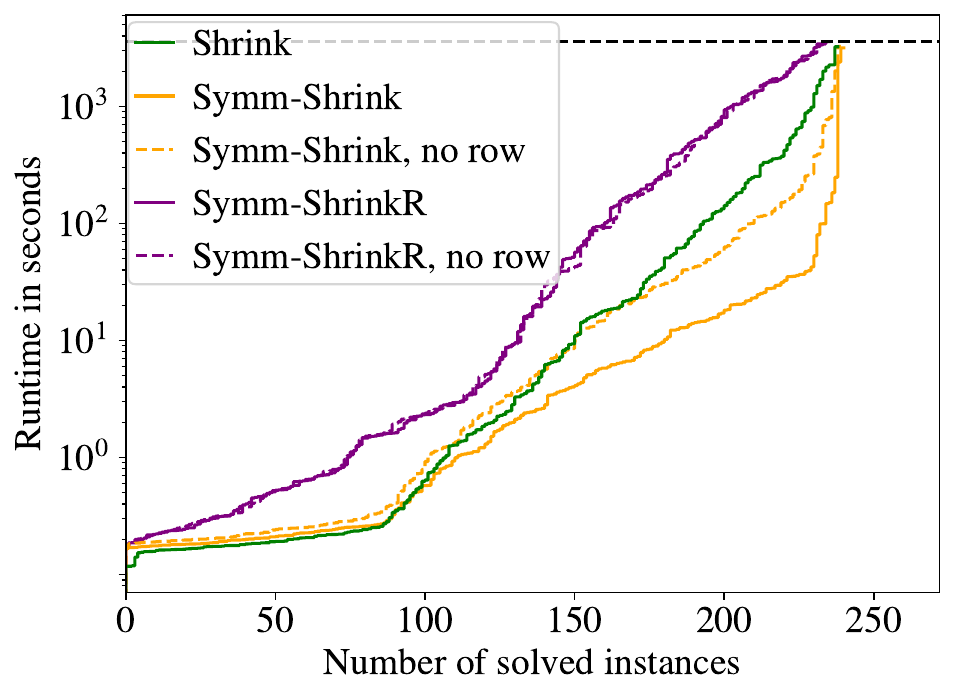}
\subcaption{Runtime for shrink-based methods}
\label{fig:deletion}
\end{subfigure}
\begin{subfigure}[c]{0.33\textwidth}
\includegraphics[width=\textwidth]{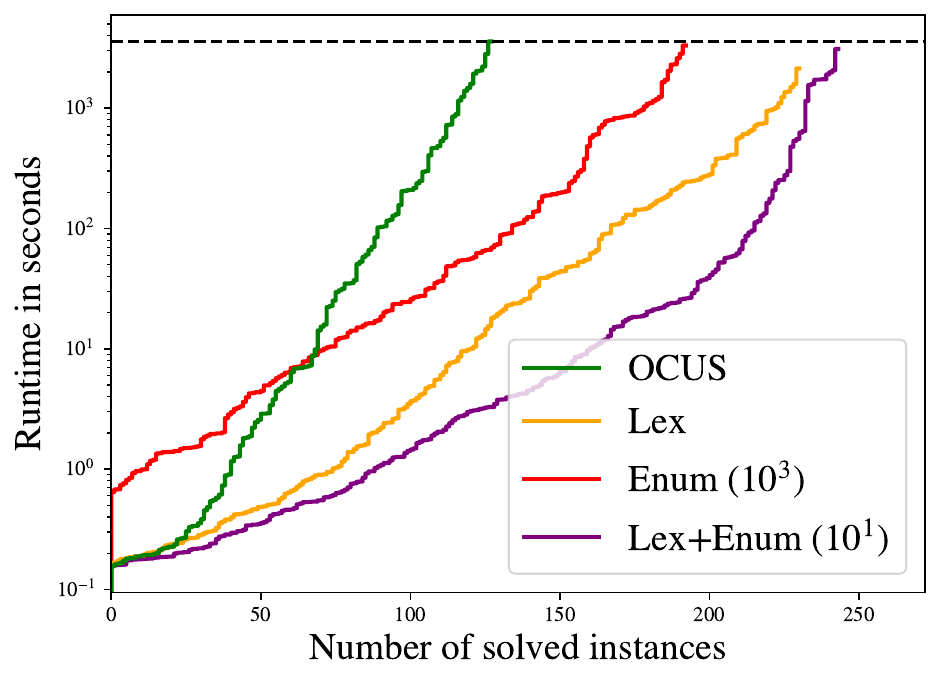}    
\subcaption{Runtime for IHS-based methods}
\label{fig:ihs}
\end{subfigure}
\begin{subfigure}[c]{0.33\textwidth}
\includegraphics[width=\textwidth]{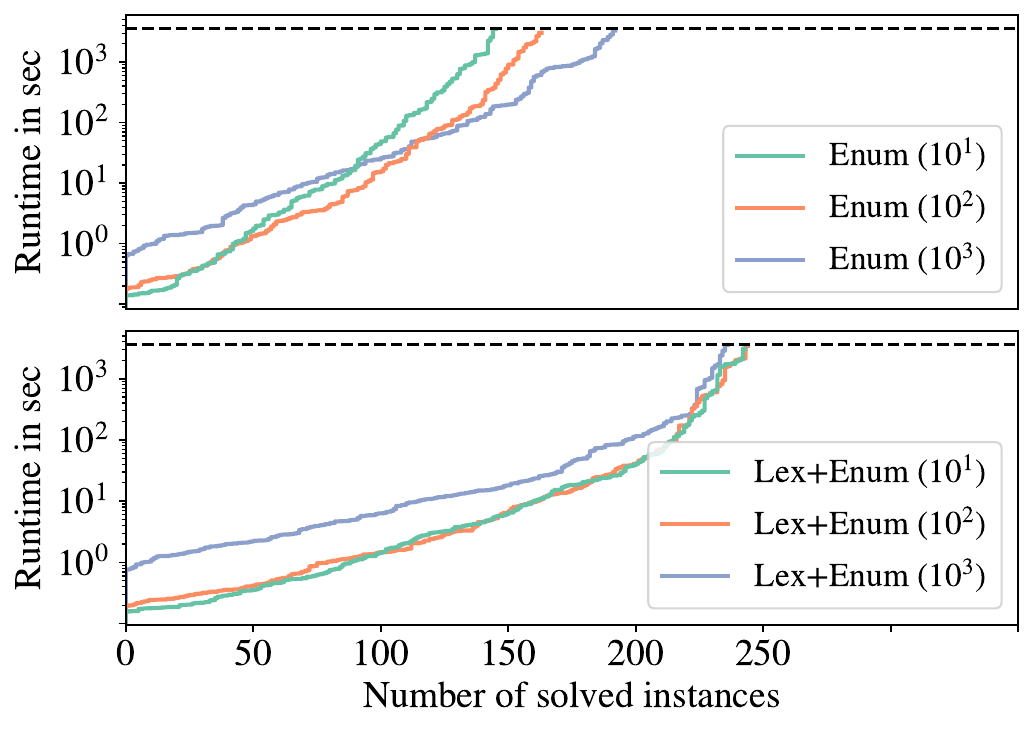}    
\subcaption{MCS enumeration for IHS-based methods}
\label{fig:mcses}
\end{subfigure}
\caption{Runtime for MUS-computation methods, measured across all 272 instances with a time-limit of 1h.}
\label{fig:computation}
\end{figure*}

In this section, we evaluate our proposed modifications to algorithms for MUS-computation and enumeration on a set of benchmark instances.
We aim to answer the following experimental questions:

\begin{EQ}
    \item To what extent is making MUS-computation symmetry-aware beneficial in terms of runtime? 
    \item How can symmetries be used for MUS enumeration?
    \item How do MUS-computation algorithms benefit from the detection of row-interchangeability symmetries?
\end{EQ}

We run all methods presented in this paper on unsatisfiable constraint problems encoded as pseudo-Boolean problems.
Our benchmark consists of 272 instances with 146 pigeon-hole problems, 66 n+k-queens problems, and 60 bin-packing problems.

We implemented\footnote{\url{https://github.com/ML-KULeuven/SymmetryMUS}} all MUS-finding algorithms on top of the \cpmpy constraint modeling library \cite{guns2019increasing}, version 0.9.20 in Python 3.10.14.
Pseudo-Boolean solver Exact v1.2.1 \cite{Exact,elffers2018divide} is used as SAT-oracle and Gurobi v11.0.2 as hitting-set-solver.
Symmetries are computed using a custom branch of \breakid\footnote{on commit \texttt{4e9b15fd}} \cite{breakid2016}.
All methods were run on a single core of an Intel(R) Xeon(R) Silver 4214 CPU with 128GB of memory on Ubuntu 20.04.
We used a time-out of 1h which includes symmetry-detection by \breakid and unrolling to symmetric MUSes in \call{Lex-Marco}.

\subsection{MUS computation}
Figure~\ref{fig:computation} shows the runtime of all methods for computing MUSes discussed in this paper.
We first focus on \Cref{fig:deletion}, which compares the runtime of shrink-based methods.
We compare the default algorithm (\sotashrink), the version removing symmetric transition constraints (\skipsym), and its version when recomputing symmetries in each iteration of the algorithm (\recomp). 
For both symmetry-aware algorithms, we also compare the runtime when instructing BreakID to detect no row-interchangeability symmetries.

We can clearly see the removal of symmetric transition constraints is beneficial for the runtime of the algorithm as \skipsym\xspace solves all instances between 5 and 10 times faster compared to the default. 
Still, the detection of row-interchangeability symmetries is essential for a successful implementation of this approach.
For most instances in our benchmarks, recomputing symmetries in each iteration of the algorithm proves to be less efficient, even compared to the default algorithm (EQ3).

Next, we compare the runtime of each version of the IHS algorithms for computing OCUSes.
We compare the default algorithm (\sotaocus), the addition of symmetry breaking constraints (\lex), the enumeration of symmetric MCSes (\enumsym($u$)), and lastly, the combination (\lex+\enumsym($u$)). Here, $u$ is the upperbound on the number of MCSes added in each iteration.
From \Cref{fig:ihs}, we see that any version of the algorithm performs better compared to OCUS, and the most efficient version uses a combination of symmetry breaking constraints and the enumeration of symmetric MCSes.
Indeed, compared to the OCUS, Lex+Enum finds an MUS for almost double the number of instances.

Comparing the different versions of Enum, we notice an increase in runtime for easier instances when adding more MCSes.
This can clearly be seen from \Cref{fig:ihs} and from the top of \Cref{fig:mcses}.
Still, Enum can solve more instances as the number of MCSes increases, but there clearly is a limit to where this method can be pushed.
When combining symmetry-breaking and MCS-enumeration, more MCSes do not yield in more instances solved, and the performance slightly degrades instead.
This can be seen from the bottom of \Cref{fig:mcses}.
This is due to the overhead in the hitting set solver when dealing with both the lex-leader constraints and the surplus in sets to hit.
Hence, for Lex+Enum, it is better to keep the number of extra MCSes low.

Overall, we can conclude that making MUS-computation techniques aware of symmetries in the unsatisfiable problem has a positive impact on their runtime (EQ1).

\subsection{MUS enumeration}
\begin{figure}[t]
    \begin{subfigure}[c]{0.49\columnwidth}
        \includegraphics[width=\textwidth]{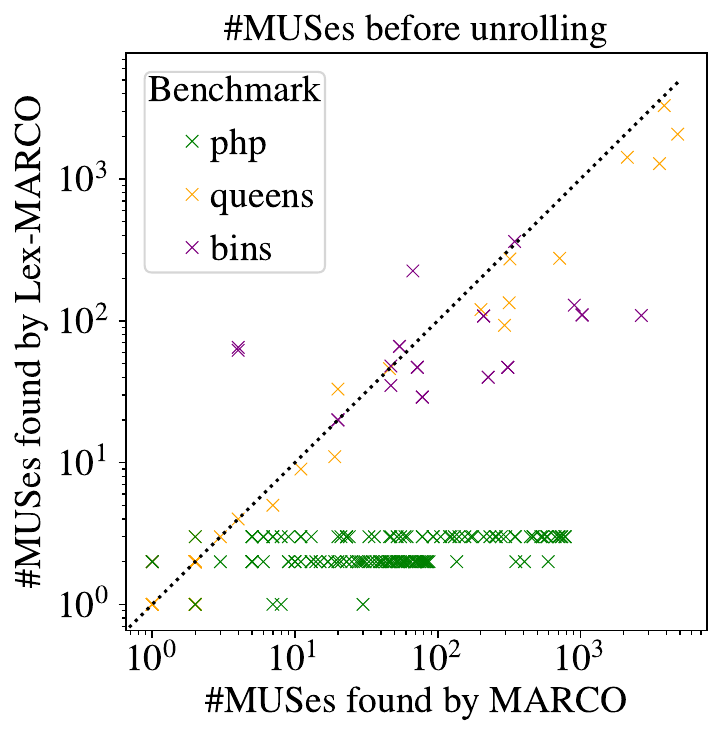}   
    \end{subfigure}
    \begin{subfigure}[c]{0.49\columnwidth}
        \includegraphics[width=\textwidth]{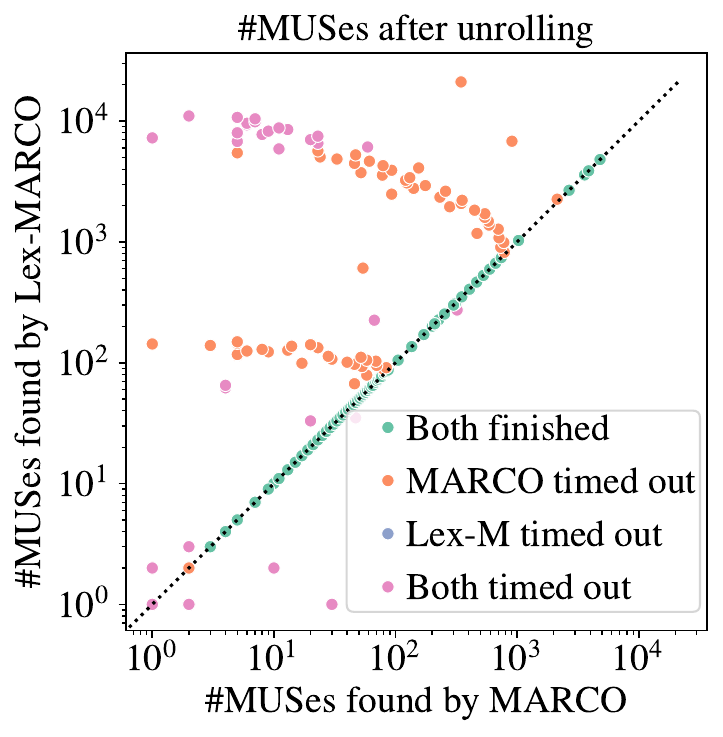}   
    \end{subfigure}
    \caption{Number of MUSes enumerated within 1h}
    \label{fig:enumeration}
\end{figure}
We use the MARCO algorithm and Lex-MARCO with post-processing to unroll the set of lex-minimal MUSes to the full set.
Both algorithms used the non-symmetric version of \call{Shrink} for shrinking and the \call{BLS} procedure from \citet{marques2013computing} for growing.
\Cref{fig:enumeration} shows the number of MUSes enumerated within 1h.
From the left-hand side, it is clear the number of MUSes computed from the lex-minimal seeds is reduced for most of the instances compared to running vanilla MARCO.
Instances found above the diagonal timed out for MARCO and the number of MUSes enumerated so far was smaller than the number of MUSes computed from the lex-minimal seeds by Lex-MARCO.
 
On the right-hand side of \Cref{fig:enumeration}, we show that the total number of MUSes after unrolling is higher compared to using the baseline MARCO algorithm.
Indeed, Lex-MARCO completely enumerates the set of MUSes for 197 instances compared to MARCO completing 112.
When both algorithms reach the timeout, MARCO may compute more MUSes as Lex-MARCO cannot start, or cannot complete unrolling within the given time-budget.

Overall, we can conclude that exploiting constraint symmetries can aid MUS enumeration by reducing the number of returned MUSes before unrolling and speeding up MUS enumeration when the full set is to be enumerated (EQ2).
\end{section}

\begin{section}{Conclusion and outlook}
In this paper, we propose various modifications to algorithms for finding MUSes inspired by symmetry breaking in SAT-and PB-solving.
We demonstrated how to use off-the-shelf symmetry-detection tools to find symmetries in constraint specifications and how to use those symmetries in MUS computation techniques.
Our results show that the presence of symmetry can indeed slow down MUS-finding algorithms, and symmetry handling methods are essential for finding MUSes quickly.
This paper opens the door for future symmetry-inspired enhancements for computing MUSes, and as a next step, we plan to implement our methods into an existing state-of-the-art MUS finder.
This allows us to evaluate our methods on competition benchmarks and compare them to other techniques, such as model rotation.
\end{section}

\section*{Acknowledgements} 

This work was partially supported by  Fonds Wetenschappelijk Onderzoek -- Vlaanderen (project G070521N), by the European Union (ERC, CertiFOX, 101122653 \& ERC, CHAT-Opt, 01002802 \& Europe Research and Innovation program TUPLES, 101070149). Views and opinions expressed are however those of the author(s) only and do not necessarily reflect those of the European Union or the European Research Council. Neither the European Union nor the granting authority can be held responsible for them.

\bart{This refs file seems to be somewhat messy. For instance multiple copies of BreakID paper with different details. Needs to be taken care of before CR }
\bibliography{aaai25-symmetrymus.bib}

\ifarxiv
\begin{appendix}

\section{Proofs}
\label{appendix:proofs}

\setcounter{proposition}{0}

\begin{proposition}[Deriving partial symmetries]
Let $\subsetT$ be a subset of $\alllits(\formula)$ closed under negation and $\pi$ a symmetry of $\formula$.
If $\pi(\subsetT) = \subsetT$, then $\pi|_\subsetT$ is a partial \subsetT-symmetry of $\formula$.   
\end{proposition}
\begin{proof}
Let $\mu$ be any assignment of $S$ (which is a partial assignment of the set of all variables). 
We need to show that  $\mu \approxmodels \formula$ iff $\pi(\mu) \approxmodels \formula$.
First assume  $\mu \approxmodels \formula$. In this case there is an $\alpha$ that extends $\mu$ and such that $\alpha\models \formula$. But this means that $\pi(\alpha)$ is a complete assignment that extends $\pi(\mu)$ and since $\pi$ is a symmetry, $\pi(\mu)\models \formula$, so indeed $\pi(\mu) \approxmodels \formula$.
For the other direction, if  $\pi(\mu) \approxmodels \formula$ there is an $\alpha$ that extends $\pi(\mu)$ and satisfies $\varphi$. But then $\pi^{-1}(\alpha)$ extends $\mu$ and also satisfies $\varphi$, which indeed shows that $\mu \approxmodels \formula$.
%
%
%
\end{proof}

\begin{proposition}
Let $\pi$ be a permutation of variables (that is a permutation of literals that does not cross polarity) $\indicators$ in $\indicators \implies \formula$, then $\pi$ directly maps to a permutation of constraints $\pi_\formula$.
Under this mapping, $\pi$ is an $\indicators$-symmetry iff $\pi_\formula$ is a constraint symmetry.
\end{proposition}

\begin{proof}
Let $\mu$ be an assignment of $\indicators$ and let $C_\mu \subseteq \formula$ be the set of constraints activated by $\mu$.
If $\pi$ does not cross polarity, then each positive literal in $\mu$ is mapped by $\pi$ to exactly one other positive literal in $\mu$, thus $\pi(\mu)$ activates a new set of constraints $\pi_\formula(C_\mu)$.
If $\pi$ is a an $\indicators$-symmetry, $\mu$ can be extended to a full assignment iff $\pi(\mu)$ can. 
Hence, $C_\mu$ is satisfiable iff $\pi_\formula(C_\mu)$ is satisfiable, ensuring $\pi_\formula$ is a constraint symmetry.

For the reverse mapping: if $\pi_\formula$ is a constraint symmetry, and $C \subseteq \formula$ is satisfiable, then clearly the partial assignment of $\mu$ and $\pi(\mu)$ can be extended to a full assignment, namely a solution to $C$ and $\pi(C)$ respectively.
If $C$ is unsatisfiable, then $\mu$ and $\pi(\mu)$ cannot be extended to a full solution.

\ignace{Not sure about the reverse? Seems trivial: if }

\ignace{TODO}
\end{proof}

\begin{proposition}
Given a constraint symmetry $\pi$ of $\formula$ and a subset $\core \subseteq \formula$.
Then, $\pi(\core)$ is an MUS of \formula if and only if $\core$ is an MUS of \formula.
\end{proposition}

\begin{proof}
We need to proof $\pi(\core)$ is unsatisfiable if and only if $\core$ is unsatisfiable and; $\pi(\core)$ is subset minimal iff $\core$ is subset-minimal.
The first condition follows directly from the definition of a constraint symmetry.
We proof the second condition using contradiction.
Suppose $\core$ is an MUS for $\formula$, and $\pi(\core)$ is unsatisfiable, but not subset-minimal.
Then, we can shrink $\pi(\core)$ to a true MUS $\core' \subsetneq \pi(\core)$.
Now, it holds that $\pi^{-1}(\core') \subsetneq \core$ and by definition of a constraint symmetry, $\pi^{-1}(\core')$ is unsatisfiable, which is in contrast with the hypothesis that $\core$ is subset-minimal.
\end{proof}

\section{Example of a row-interchangeability symmetry}
\label{appendix:rowsym}
Consider the row-interchangeability symmetry as defined by the following matrix.

$$
\begin{bmatrix}
\color{blue}{a_{11}} & \color{blue}{a_{12}} & a_{13}\\
\color{blue}{a_{21}} & \color{blue}{a_{22}} & a_{23}\\
a_{31} & \color{blue}{a_{32}} & a_{33} \\
\color{blue}{a_{41}} & \color{blue}{a_{42}} & \color{blue}{a_{43}} \\
\end{bmatrix}
$$
When executing \call{Symm-Shrink}, consider the current core to be $\core = \{a_{11}, a_{12}, a_{21}, a_{22}, a_{32}, a_{41}, a_{42}, a_{43}\}$ and a transition constraint corresponding to the indicator variable $a_{12}$.
Then, the only symmetry mapping $\core$ to $\core$ defined by the matrix is swapping the first and second rows.
Hence, the only symmetric transition constraint is the one corresponding to $a_{22}$.

\section{Benchmarks}\label{appendix:benchmarks}
In this section, we provide the detailed description of the benchmarks used in this paper.
All benchmarks are derived from those used in the literature, with minor modifications in order to make the models unsatisfiable.

\paragraph{Pigeon hole problem}
We generated instances of saturated pigeon-hole problems, ranging from 5 to 150 pigeons with in increments of 2.
For each number of pigeons, the number of holes were limited to \#pigeons-2 and \#pigeons-3.
This results in a total benchmark set of 146 instances.
\begin{align*}
    & \sum_{j = 1}^{h} x_{ij} \geq 1 & \forall i = 1..p \tag{\pigeoncons{i}} \\
    & \sum_{i = 1}^{p} x_{ij} \leq 1 & \forall j = 1..h \tag{\holecons{j}}
\end{align*}

\ignore{
\paragraph{Nurse-rostering}
The nurse-rostering problem considers the assignment of $n$ nurses  to shifts in a hospital.
Each day in the time horizon $d$, $s$ shifts have to be filled by exactly $c$ nurses.
The model uses Boolean  variables $x_{ijk}$ indicating that on day $i$ in the schedule, shift $j$ is covered by nurse $k$ and contains the following constraints.

\begin{itemize}
    \item Each shift requires exactly $c$ nurses
    $$ \sum_{k = 1}^n x_{ijk} = c \quad, \forall i = 1..d, j = 1..s $$
    \item Nurses cannot do two shifts in a day
    $$\sum_{j = 1}^s x_{ijk} \leq 1 \quad, \forall i = 1..d, k = 1..n$$
    \item Nurses cannot do the last shift of a day, and the first shift of the proceeding day
    $$\neg x_{isk} \vee \neg x_{i+1,1,k} \quad, \forall i = 1..d-1, k = 1..n$$    
    
\end{itemize}
We generate 92 instances with the number of nurses ranging from 10 to 100, in increments of 2 and the time-horizon is put on either 7 of 14 days. 
For all instances, the number of shifts in a day is set to 3 and the problem is made unsatisfiable by requiring too many nurses $c$ for each shift.
}

\paragraph{N+k-queens}
The n-queens problem is a classic problem when investigating symmetries in constraint satisfaction problems.
The problem consists of finding an arrangement of $n$ queens on a $n \times n$ dimensional chess-board, in such a way that no two queens can attach each other.
In the unsatisfiable variant, we enforce $k$ queens extra on the board.
We generated 66 instances, with $n$ ranging from 4 to 25 and $k$ set to 2, 3 or 4.

\begin{align*}
    &\sum_{i = 1}^n x_{ij} \leq 1 \quad, \forall j = 1..n \\
    &\sum_{j = 1}^n x_{ij} \leq 1 \quad, \forall i = 1..n \\
    & \sum_{i=1}^{n-j+1} x_{i+j-1,i} \leq 1 \quad, \forall j=1..n\\
    & \sum_{i=1}^{n-j+1} x_{n-i-j+2,i} \leq 1 \quad, \forall j=1..n\\
    & \sum_{i=1}^{n-j+1} x_{i+j-1,n-i+1} \leq 1 \quad, \forall j=2..n\\
    & \sum_{i=1}^{n-j+1} x_{n-i-j+2,n-i+1} \leq 1 \quad, \forall j=2..n\\
    &\sum_{i = 1}^n \sum_{j = 1}^n x_{ij} = n + k
\end{align*}

\paragraph{Binpacking}
Bin-packing considers the problem of packing $n$ items with a given value (size) into $b$ equivalent bins with a fixed capacity $c$.
The goal is to find an allocation of items to bins such that the number of used bins is minimized.
We generate 60 instances similar to the approach in \cite{schwerin1997bin}.
We range the number of bins from 5 to 25 and sample the size of each item from $\mathcal{U}(5,10)$.
The capacity of each bin is fixed to 25 and the problems are made unsatisfiable by limiting the number of available bins to 70, 80 or 90\% of the required number of bins.
\begin{align*}
    \sum_{i = 1}^n w_{ij} \cdot x_{ij} \leq c \quad, \forall j = 1..b
\end{align*}

\end{appendix}

\fi

\end{document}